\documentclass[11pt,letterpaper]{article}

\usepackage{graphicx} % more modern
\usepackage{subfigure} 
\usepackage{natbib}
\usepackage{paralist}
\usepackage{algorithm}
\usepackage{algorithmic}
\usepackage{amsmath,amssymb,xspace,amsthm,Tabbing,bm}
\usepackage{color,mathtools}
\usepackage{listings}
\usepackage{multirow}
\usepackage{fullpage}
\def \fro {\mathrm{F}}
\def\xray{\textsc{Xray} }
\def\xraykl{\textsc{Xray}-KL }
\def\xrayis{\textsc{Xray}-IS }
\def\xrayl2{\textsc{Xray}-$\ell_2$ }
\def\hott{\emph{Hottopixx} }
\newcommand{\punt}[1]{}

\newtheorem{thm}{Theorem}[section]
\newtheorem{lem}{Lemma}[section]
\def\argmax{\mathop{\rm arg\,max}}
\def\argmin{\mathop{\rm arg\,min}}

\newcommand{\reals}{\mathbb{R}}

\newcommand{\cp}{\mathcal{P}}

\newcommand{\cS}{\mathcal{S}}
\def\argmax{\mathop{\rm arg\,max}}
\def\argmin{\mathop{\rm arg\,min}}
\newcommand{\vX}{\mathbf{X}}
\newcommand{\vY}{\mathbf{Y}}

\newcommand{\vB}{\mathbf{B}}

\newcommand{\vD}{\mathbf{D}}

\newcommand{\vL}{\mathbf{L}}
\newcommand{\vI}{\mathbf{I}}
\newcommand{\vC}{\mathbf{C}}

\newcommand{\vN}{\mathbf{N}}

\newcommand{\vR}{\mathbf{R}}
\newcommand{\vS}{\mathbf{S}}

\newcommand{\vZ}{\mathbf{Z}}
\newcommand{\vW}{\mathbf{W}}
\newcommand{\vH}{\mathbf{H}}

\newcommand{\vLam}{\bm{\Lambda}}
\newcommand{\vLambda}{\bm{\Lambda}}

\newcommand{\vw}{\boldsymbol{w}}
\newcommand{\vx}{\mathbf{x}}

\newcommand{\vdelta}{\boldsymbol{\delta}}

\newcommand{\vh}{\boldsymbol{h}}

\newcommand{\vz}{\boldsymbol{z}}

\newcommand{\vv}{\mathbf{v}}

\newcommand{\vu}{\mathbf{u}}
\newcommand{\vp}{\mathbf{p}}

\newcommand{\bq}{\begin{equation}}
\newcommand{\eq}{\end{equation}}
\newcommand{\ba}{\begin{eqnarray}}
\newcommand{\ea}{\end{eqnarray}}

\newcommand{\mbf}[1]{\mathbf{#1}}
\newcommand{\mbb}[1]{\mathbb{#1}}
\newcommand{\mcal}[1]{\mathcal{#1}}
\newcommand{\remove}[1]{}

\newcommand{\Cpp}{C\kern-0.05em\texttt{+\kern-0.03em+}}

\newcommand{\ConceptCpp}{ConceptC\kern-0.05em\texttt{+\kern-0.03em+}}
\title{Near-separable Non-negative Matrix Factorization with $\ell_1$- and Bregman Loss Functions}
\author{ {Abhishek Kumar} \hspace{8mm} {Vikas Sindhwani} \\ \{abhishk,vsindhw\}@us.ibm.com
 \\ \\ IBM T.J. Watson Research Center, Yorktown Heights, NY 10598 USA
            }
\date{}

\begin{document}

%\twocolumn[
%\icmlauthor{{Abhishek Kumar}}{abhishek@cs.umd.edu}
%\icmladdress{Dept. of Computer Science, University of Maryland, College Park, MD 20742, USA}
%\icmlauthor{Vikas Sindhwani}{vsindhw@us.ibm.com}
%\icmladdress{IBM T.J. Watson Research Center,
%           Yorktown Heights, NY 10598 USA}
%\icmlauthor{Prabhanjan Kambadur}{pkambadu@us.ibm.com}
%\icmladdress{IBM T.J. Watson Research Center,Yorktown Heights, NY 10598 USA}
%\icmladdress{$^{\dagger}$Dept. of Computer Science, University of Maryland, College Park, MD 20742 USA}
%\vskip -3mm
%\icmladdress{$^{\ddagger}$IBM T.J. Watson Research Center, Yorktown Heights, NY 10598 USA}
%\icmlkeywords{separability assumption, non-negative matrix factorization, Bregman divergence, foreground-background separation in video, exemplar selection}
\maketitle
%\vskip 0.3in
%]

\begin{abstract}
Recently, a family of tractable NMF algorithms have been proposed under the
assumption that the data matrix satisfies a separability
condition~\cite{DonohoStodden,arora.stoc12}. Geometrically, this
condition reformulates the NMF problem as that of finding the extreme rays
of the conical hull of a finite set of vectors.
 In this paper, we develop several extensions of the conical hull procedures
of~\cite{xray.icml13} for robust ($\ell_1$)
approximations and Bregman divergences. Our methods inherit all the advantages
of~\cite{xray.icml13} including scalability and noise-tolerance. We show that
on foreground-background separation problems in computer vision, robust
near-separable NMFs match the performance of Robust PCA, considered state of the
art on these problems, with an order of magnitude faster training time. We also
demonstrate applications in exemplar selection settings.
 
\end{abstract}

\section{Introduction}
\def\xray{\textsc{Xray}}
The problem of non-negative matrix factorization (NMF) is to express a non-negative matrix $\vX$ of size $m\times n$, either exactly
or approximately, as a product of 
two non-negative matrices, $\vW$ of size $m\times r$ and $\vH$ of size $r\times
n$. Approximate NMF attempts to minimize a  measure of divergence between the
matrix $\vX$ and the factorization $\vW\vH$.
The inner-dimension of the factorization $r$ is usually taken to be much smaller than $m$ and $n$ to get interpretable part-based representation
of data~\cite{LeeSeung}. NMF is used in a wide range of applications, e.g., topic modeling and text mining, hyper-spectral image analysis, 
audio source separation, and microarray data analysis~\cite{nmf:book}.

The exact and approximate NMF problem is NP-hard. Hence, traditionally,
algorithmic work in NMF has focused on treating it as an instance of non-convex
optimization~\cite{nmf:book,LeeSeung,cjlin.nmf07,dhillon.kdd11} leading to
algorithms lacking optimality guarantees beyond convergence to a stationary
point. 
% Geometrically, we can think of the columns of matrices $\vX$ and $\vW$ as points in $\reals^m$. 
% For any matrix $\vA$, let $cone(\vA)$ denote the set of points generated by taking non-negative linear combinations of
% columns of $\vA$. The problem of exact NMF ($\vX=\vW\vH$) 
% can be understood as that of finding a non-negative matrix $\vW$ such that the cone generated by its $r$ columns
% contains all the columns of $\vX$: $$cone(\vX)\subseteq cone(\vW)\subseteq
% \reals_+^m,$$ where $\reals_+^m$ is the non-negative orthant in $\reals^m$. 
% These polyhedral nesting problems are known to be NP-hard which makes exact and approximate NMF problem also NP-hard~\cite{Vavasis.09}.
% Although it was formally shown only in $2009$ \cite{Vavasis.09}, the NMF problem was believed to be NP-hard for a long time
% and hence almost the entire body of algorithmic work in NMF has focused on
% treating it as a non-convex optimization problem \cite{nmf:book,LeeSeung,cjlin.nmf07,dhillon.kdd11}, leading to heuristic local search based procedures. 
% These algorithms lack optimality guarantees beyond convergence to a stationary point 
% of the approximate NMF objective function and are sensitive to initialization. 
Promising alternative approaches have emerged recently based on a
\emph{separability} assumption on the data
\cite{arora.stoc12,bittorf.12,gillis.12,xray.icml13,esser.11} which enables the
NMF problem to be solved efficiently and exactly.
Under this assumption, the data matrix $\vX$ is said to be $r$-separable if all
columns of $\vX$ are contained in the conical hull generated by a subset of $r$
columns of $\vX$.  In other words, if $\vX$ admits a factorization $\vW\vH$
%(exactly or approximately) 
then the separability assumption states that the columns of $\vW$ are present in $\vX$ at positions
given by an unknown index set $A$ of size $r$. Equivalently, the corresponding
columns of the right factor matrix $\vH$ constitute the $r\times r$ identity matrix, i.e.,
$\vH_A = \vI$. We refer to these columns indexed by $A$ as \emph{anchor
columns}.

The separability assumption was first investigated by \citet{DonohoStodden} in
the context of deriving conditions for uniqueness of NMF.
NMF under separability assumption has been studied for topic modeling in text
~\cite{xray.icml13,arora2012practical} and hyper-spectral
imaging~\cite{gillis.12,esser.11}, and separability has turned out to be a
reasonable assumption in these two applications. In the context of topic
modeling where $\vX$ is a document-word matrix and $\vW$, $\vH$ are
document-topic and topic-word associations respectively, it translates to
assuming that there is at least one word in every topic that is unique to itself
and is not present in other topics.

Our starting point in this paper is the family of conical hull finding
procedures called~\textsc{Xray} introduced
in~\cite{xray.icml13} for near-separable NMF problems with Frobenius norm
loss.~\xray~finds anchor columns one after the other, incrementally expanding
the cone and using exterior columns to locate the next anchor.~\xray~has several
appealing features: (i) it requires no more than $r$ iterations each of which is
parallelizable, (ii) it empirically demonstrates  
noise-tolerance, (iii) it admits efficient model selection, and (iv) it does not
require normalizations or preprocessing needed in other methods. However, in the
presence of outliers or different noise characteristics, the use of Frobenius
norm approximations is not optimal.  

\begin{figure}
\begin{center}
\includegraphics[clip=true,trim=0 105mm 0
0,width=0.7\linewidth,height=2.5cm]{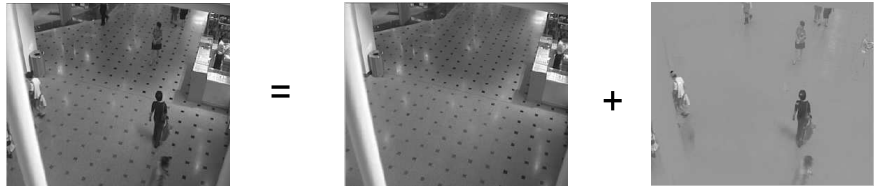}   
\end{center}
\vspace{-5mm}
\caption{Robust\xray~applied to video background-foreground separation
problem}\label{fig:mall}
\vspace{-3mm}
\end{figure}

In this paper, we extend~\xray~to provide robust factorizations with respect to
$\ell_1$ loss, and approximations with respect to the family of Bregman
divergences.  Figure~\ref{fig:mall} shows a motivating application from computer
vision. Given a sequence of video frames, the goal is to separate a
near-stationary background from the foreground of moving objects that are
relatively more dynamic across frames but span only a few pixels. In this
setting, it is natural to seek a low-rank background matrix $\vB$ that minimizes
$\|\vX - \vB\|_1$ where $\vX$ is the frame-by-pixel video matrix, and the $\ell_1$
loss imposes a sparsity prior on the residual foreground.
Unlike the case of low-rank approximations in Frobenius or spectral norms, this
problem does not admit an SVD-like tractable solution.  The Robust Principal
Component Analysis (RPCA), considered state of the art for this application,
uses a nuclear-norm convex relaxation of the low-rank constraints. In this
paper, we instead recover tractability by imposing the separable NMF assumption
on the background matrix. This implies that the variability of pixels across the
frames can be "explained" in terms of observed variability in a small set
of anchor pixels. Under a more restrictive setting, this can be shown to be 
equivalent to median filtering on the video frames, while a full near-separable
NMF model imparts more degrees of freedom to model the background. We show that the
proposed near-separable NMF algorithms with $\ell_1$ loss are competitive with RPCA in
separating foreground from background while outperforming it in terms of computational efficiency.

Our algorithms are empirically shown to be robust to noise (deviations from the
pure separability assumption). In addition to the background-foreground problem,
we also demonstrate our algorithms on the exemplar selection problem. 
%Further, we also show their applicability in two novel problems that have not
% been explored yet in the context of separable NMFs:
%(i) identifying exemplars or representatives from a collection of samples, and
%(ii) separating foreground from background in video. 
For identifying exemplars in a data set, the proposed algorithms 
are evaluated on text documents with classification accuracy as a performance metric and are shown
to outperform the recently proposed method of \citet{elhamifar.cvpr12}. 
%In the problem of foreground-background separation in video, the background is
% modeled as almost stationary or varying very slowly across the video frames while the foreground is modeled as objects that are relatively more dynamic across frames but span only a few pixels.  
%For foreground-background separation, we present our separable NMF algorithms
% with $\ell_1$ loss as a promising alternative to Robust Principle Component Analysis (RPCA) which is a convex relaxation 
%of the NP hard problem $\min_{\vL,\vS} \text{rank}(\vL) + \lVert
% \vS\rVert_0,\,\,\text{s.t. } \vX = \vL+\vS$~\cite{} and is a popular approach for this task~\cite{}. 
%NMF can be used to address this problem where we can model the background $L$ as a low non-negative rank matrix. Since NMF is NP-hard, we advocate proposed methods
%based on separability assmption for making it tractable. 

{\bf Related Work:} Existing separable NMF methods work either with only a
limited number of loss functions on the factorization error 
such as Frobenius norm loss \cite{xray.icml13}, $\ell_{1,\infty}$ norm loss \cite{bittorf.12}, or maximize proxy criteria 
such as volume of the convex polyhedron with anchor columns as vertices \cite{gillis.12} and 
distance between successive anchors \cite{arora2012practical} to select the anchor columns. 
On the other hand, local search based NMF methods \cite{nmf:book} have been proposed for a wide variety of loss functions on
the factorization error including $\ell_1$ norm loss
\cite{manhattannmf.12,sparsenmf} and instances of Bregman divergence
\cite{bregmannmf.12,nmfbregman05}.
In this paper, we close this gap and propose algorithms for near-separable NMF that minimize $\ell_1$ loss and Bregman divergence 
for the factorization.  

%{\bf Related Work:} \\
%-- briefly mention the related work and highlight the contributions again

%\include{related}

\section{Geometric Intuition}
\label{sec:geom}
The goal in exact NMF is to find a matrix $\vW$ such that
the cone generated by its columns (i.e., their non-negative linear combinations)
contains all columns of $\vX$.
Under separability assumption, the columns of matrix $\vW$ are to be picked directly from $\vX$, also known as anchor columns. 
The algorithms in this paper build the cone incrementally by picking a column from $\vX$ in every iteration. The algorithms execute
$r$ such iterations for constructing a factorization of inner-dimension $r$. Figure~\ref{fig:cone2} shows the cone after three
iterations of the algorithm when three anchor columns have been identified. An extreme ray $\{t\vx : t>0\}$ is associated with 
every anchor point $\vx$. The points on an extreme ray cannot be expressed as
conic combinations of other points in the cone that do not themselves lie on that extreme ray. 
To identify the next anchor column, the algorithm
picks a point outside the current cone (a green point) and projects it to the current cone so that the distance between
the point and the projection is minimized in terms of the desired measure of
distance. This projection is then used to setup a specific simple criteria which
when maximized over the data points, identifies a new anchor. This new anchor 
is then added to the current set of anchors and the cone is expanded
iteratively until all anchors have been picked.
%A similar framework consisting of selection and projection steps is used in the \xray family of algorithms \citet{xray.icml13} 
%but the algorithms there are designed for Gaussian i.i.d. perturbations to the separable structure 
%(hence, the Frobenius norm on the error: $\lVert \vX-\vX_A\vH\rVert_\fro$) and are not suitable for other type of noise distributions
%as we show later in Sec.~\ref{empirical}.

\begin{figure}[t]
%\vspace{-3mm}
\begin{center}
\includegraphics[clip=true,trim=0 2mm 0 2.5cm,height=4cm, width=0.7\linewidth]{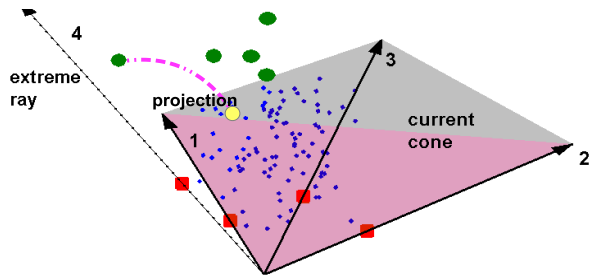}
\label{fig:cone2}
\vspace{-5mm}
\caption{Geometrical illustration of the algorithm} 
\end{center}
%small blue points are inside the current cone, larger green points are outside the current cone,
%square red points are selected anchors. In the next iteration, exterior points are projected to the current cone. Yellow point is the projection
%of an exterior point in terms of the chosen distance measure. A fourth anchor point is selected using these projections and the cone is expanded by adding
%the corresponding extreme ray.}
\vspace{-4mm}
\end{figure} 

%clarkson-dula
These geometric intuitions are inspired by \citet{clarkson,dula} who present linear programming (LP) based algorithms for 
general convex and conical hull problems. Their algorithms use $\ell_2$ projections 
of exterior points to the current cone and are also applicable in our NMF setting if the data matrix
$\vX$ satisfies $r$-separability exactly. In this case, the $\ell_2$ projection and corresponding residual vector of {\it any} single exterior point 
can be used to expand the cone and {\it all} $r$ anchors will be recovered correctly at the end of the algorithm. 
When $\vX$ does not satisfy $r$-separability exactly, anchor selection criteria derived from {\it multiple} residuals demonstrate 
superior noise robustness as empirically shown by \citet{xray.icml13} who consider the case of Gaussian i.i.d. noise. 
However,  
the algorithms of \citet{xray.icml13} are not suitable for noise distributions other than Gaussian (e.g., other members of the exponential family, sparse noise)
as they minimize $\lVert \vX-\vX_A\vH\rVert_\fro^2$. In the 
following sections, we present algorithms for near-separable NMF that are targeted precisely towards this goal and empirically demonstrate 
their superiority over existing algorithms under different noise distributions. 

%%%%%%%%%%%%%%%%%%%%%%%%%%%%%%%%%%%%%%%%%%%%%%%%%%%%%%%%%%%%%%%%%%%%%%%%%%%%%%%%%

\section{Near-separable NMF with $\ell_1$ loss}
\label{sec:l1}
This section considers the case when the pure separable structure is perturbed by sparse noise. 
Hence our aim is to minimize $\lVert\vX-\vX_A\vH\rVert_1$ for $\vH\geq 0$ where $\lVert\cdot\rVert_1$ denotes element-wise
$\ell_1$ norm of the matrix and $\vX_A$ are the columns of $\vX$ indexed by set $A\subset\{1,2,\ldots,n\}$.
We denote $i$th column of $\vX$ by $\vX_i$. 
The proposed algorithm proceeds by identifying one anchor column in each iteration and adding it
to the current set of anchors, thus expanding the cone generated by anchors. Each iteration consists of two steps:
(i) \emph{anchor selection} step that finds the column of $\vX$ to be added as an anchor, and
(ii) a \emph{projection} step where all data points (columns of $\vX$) are projected to the current cone
in terms of minimizing the $\ell_1$ norm. 
Algorithm~\ref{alg:conicl1} outlines the steps of the proposed algorithm. 
 
{\bf Selection Step}: In the selection step, we normalize all the points to lie
on the hyperplane $\vp^T\vx=1$ ($\vY_j = \frac{\vX_j}{\vp^T\vX_j}$) for a strictly positive vector $\vp$ and evaluate the selection criterion of Eq.~\ref{eq:selection}
to select the next anchor column. Note that any exterior point ($i: \|\vR_i\|_1\geq 0$) can be used in the selection criterion
-- Algorithm~\ref{alg:conicl1} shows two possibilities for choosing the exterior point. Taking the point
with maximum residual $\ell_1$ norm to be the exterior point turns out to be far more robust to noise than randomly
choosing the exterior point, as observed in our numerical simulations. 
 
{\bf Projection Step}: The projection step, Eq.~\ref{eq:nnlad}, involves solving
a multivariate least absolute devitations problems with non-negativity
constraints. We we use alternating direction method of
multipliers (ADMM) \cite{boyd.admm} and reformulate the problem as $$\min_{B\geq
0,\vZ} \|\vZ\|_1,\text{ such that}~\vX_A\vB+\vZ=\vX.$$
Thus the non-negativity constraints are decoupled from the $\ell_1$ objective and the ADMM optimization
proceeds by alternating between two sub-problems -- a standard $\ell_1$ penalized $\ell_2$ proximity problem in variable $\vZ$ which has a closed
form solution using the soft-thresholding operator, and a non-negative least squares problem in variable $\vB$ that is solved using 
a cyclic coordinate descent approach (cf. Algorithm 2 in \cite{xray.icml13}). The standard primal and dual
residuals based criteria is used to declare convergence~\cite{boyd.admm}. The ADMM procedure converges to the global optimum
since the problem is convex. 

\def\conic{}
\begin{center}
\begin{algorithm}[tb]
\caption{Robust\xray: Near-separable NMF with $\ell_1$ loss}
\label{alg:conicl1}
\begin{algorithmic}
{\small 
\STATE {\bfseries Input:} $\vX \in \reals_+^{m\times n}$, inner dimension $r$ 
%an appropriate selection operator $\mcal{S}: \mbb{R}^n \mapsto \mbb{R}$
\STATE {\bfseries Output:} $\vW\in \reals_{m\times r}, \vH \in \reals_{r\times n}$, $r$ indices in $A$ \\
~~~~~~~~~~~~such that: $\vX = \vW \vH$, $\vW = \vX_A$ 
\STATE {\bfseries Initialize:} $\vR \gets \vX$, $\vD^\star \gets \vX$, $A\gets \{\}$
%\FOR{$t=1$ \ldots $r$}
\WHILE{$|A|<r$} 
\STATE 1. {\it {\bfseries Anchor Selection step}:} %{\it Find an extreme ray}}.
%Let $\vR_i$ be the residual of a point $\vX_i$ that lies outside the current cone, i.e., $\vR_i\neq 0$. %$j^* = \argmax_j \mcal{S}\left(\frac{\vR^T \vX_j}{\vp^T \vX_j}\right)$
\STATE First, pick any point exterior to the current cone. Two possible criteria are
\vspace{-2mm}
\begin{eqnarray}
\textrm{\textit{rand}}:  & ~~\textrm{any random}~~i : \|\vR_i\|_1>0\label{eq:rand}\\
\textrm{\textit{max}}:   & i = \argmax_k \|\vR_k\|_1\label{eq:max}
\end{eqnarray}
\vspace{-4mm}
\STATE Choose a suitable $\vD_i^\star\in\vD_i$ where $\vD_{ji} = sign(\vR_{ji})$ if $\vR_{ji}\neq 0$, else $\vD_{ji} \in [-1,1]$ (see Remark (1)).
\STATE  Select an anchor as follows ($\vp$ is a strictly positive vector, not collinear with $\vD_i^\star$ (see Remark (3))):
\vspace{-3mm}
\begin{equation}
j^\star = \argmax_j \frac{\vD_i^{\star^T} \vX_j}{\vp^T \vX_j}%~~\textrm{for any}~~i:\|\vR_i\|_1>0
\label{eq:selection}
\end{equation}
\vspace{-2mm}
%\begin{equation} \textrm{A \textit{greedy} variant:}~j^* = \argmax_j \frac{\|(\vR^T \vX_j)_+\|^2}{\|\vX_j\|^2_2} \label{eq:greedy}\end{equation} 
% where $\vp$
%is a strictly positive vector. % such that $\vp$ and $\vR_i$ are linearly independent. 
\STATE 2. Update: $A\gets A\cup \{j^\star\}$ (see Remark (2))
%\STATE If maximum is obtained at two points $\vX_{j_1^*}$ and $\vX_{j_2^*}$, both these points are the extreme rays
%of the cone. Update $A\gets A\cup \{j_1^*, j_2^*\}$.
%\STATE If maximum is obtained at more than two points, identified by the index set $J = \{j_1^*,j_2^*,\ldots,j_k^*\}$,
%find the extreme rays of these points by recursively calling this algorithm with input $\vX_J$. Add these extreme rays
%to the set $A$. 
\vspace{2mm}
\STATE 3. {\it {\bfseries Projection step}: Project onto current cone.} 
\begin{equation}\vH = \argmin_{\vB\geq 0}\|\vX - \vX_A \vB\|_1~~~~\text{(ADMM)}\label{eq:nnlad}\end{equation}
\vspace{-3mm}
\STATE 4. Update Residuals:  ~~ $\vR = \vX - \vX_A \vH$
\ENDWHILE
}
\end{algorithmic}%\vspace{-0.5cm}
\end{algorithm}
%\vspace{-1cm}
\end{center}
We now show that Algorithm~\ref{alg:conicl1} correctly identifies all the
anchors in pure separble case.
\begin{lem}
Let $\vR$ be the residual matrix obtained after $\ell_1$ projection of columns of $\vX$ onto the current cone 
and $\vD$ be the set of matrices such $\vD_{ij}=sign(\vR_{ij})$ if $\vR_{ij}\neq 0$ else $\vD_{ij}\in[-1,1]$. 
Then, there exists at least one $\vD^\star\in\vD$ such that $\vD^{\star^T} \vX_A \leq 0$, where $\vX_A$ are anchor columns selected so far by Algorithm~\ref{alg:conicl1}.
\label{lem:kktnnlad}
\vspace{-3mm}
\end{lem}
\begin{proof}
Residuals are given by $\vR = \vX - \vX_A \vH$, where $\vH = \argmin_{\vB\geq 0} \lVert \vX - \vX_A \vB\rVert_1$. \\
Forming the Lagrangian for Eq.~\ref{eq:nnlad}, we get
$\mcal{L}(\vB,\vLam) = \lVert \vX - \vX_A \vB\rVert_1 - tr(\vLam^T\vB)$, where the matrix $\vLam$ contains
the non-negative Lagrange multipliers.
The Lagrangian is not smooth everywhere and its sub-differential is given by $\partial\mcal{L}=-\vX_A^T\vD-\vLam$ where 
$\vD$ is as defined in the lemma.
At the optimum $\vB = \vH$, we have $0\in \partial\mcal{L}$
$\Rightarrow -\vLam\in\vX_A^T\vD$. Since $\vLam\geq 0$, this means that there exists at least one $\vD^\star\in\vD$
for which $\vD^{\star^T}\vX_A\leq 0$.
%As $\vX_A\geq 0$, we can easily find one such element by taking $\vD_{ij}=-1$ whenever $\vR_{ij}=0$. 
%\vspace{-1mm}
\end{proof}

\begin{lem}
For any point $\vX_i$ exterior to the current cone, there exists at least one $\vD^\star\in\vD$ such that it 
satisfies the previous lemma and $\vD_i^{\star^T} \vX_i > 0$.
\label{lem:extcone}
\vspace{-3mm}
\end{lem}
\begin{proof}
Let $\vR = \vX - \vX_A \vH$, where $\vH = \argmin_{\vB\geq 0} \lVert \vX - \vX_A \vB\rVert_1$ and $\vX_A$ 
are the current set of anchors.
From the proof of previous lemma, $-\vLam^T\in\vD^T\vX_A$.
%where $\vLam\geq 0$ are the Lagrange multipliers. 
Hence, $-\vLam_i^T\in\vD_i^T\vX_A$ ($i$th row of both left and right side matrices). From the complementary
slackness condition, we have $\vLam_{ji}\vH_{ji} = 0\,\, \forall\, j,i$. Hence, $-\vLam_i^T\vH_i = 0\in\vD_i^T\vX_A\vH_i$. \\
Since all KKT conditions are met at the optimum, there is at least one $\vD^\star\in\vD$ that satisfies previous lemma
and for which $\vD_i^{\star^T}\vX_A\vH_i=0$. 
For this $\vD^\star$, we have $\vD_i^{\star^T} \vX_i = \vD_i^{\star^T} (\vR_i + \vX_A \vH_i) = \vD_i^{\star^T}\vR_i = \lVert\vR_i\rVert_1> 0$ 
since $\vR_i\neq 0$ for an exterior point. 
%\vspace{-2mm}
\end{proof}
Using the above two lemmas, we prove the following theorem regarding the correctness of 
Algorithm~\ref{alg:conicl1} in pure separable case.
\begin{thm}
If the maximizer in Eq.~\ref{eq:selection} is unique, the data point $\vX_{j^\star}$ added at each iteration in the Selection step of Algorithm~\ref{alg:conicl1}, is 
an anchor that has not been selected in one of the previous iterations.
\vspace{-3mm}
\end{thm}
\begin{proof}
Let the index set $A$ denote all the anchor columns of $\vX$. Under the separability assumption, we have $\vX = \vX_A \vH$. 
Let the index set $A^t$ identify the current set of anchors.

Let $\vY_j = \frac{\vX_j}{\vp^T\vX_j}$ and $\vY_A = \vX_A [diag(\vp^T\vX_A)]^{-1}$ (since $\vp$ is strictly positive, 
the inverse exists).
Hence $\vY_j =$ $\vY_A \frac{[diag(\vp^T\vX_A)] \vH_j}{\vp^T\vX_j}$. Let $\vC_j = \frac{[diag(\vp^T\vX_A)] \vH_j}{\vp^T\vX_j}$. 
We also have $\vp^T\vY_j = 1$ and $\vp^T\vY_A = \mbf{1}^T$. Hence, we have $1 = \vp^T\vY_j = \vp^T\vY_A \vC_j = \mbf{1}^T \vC_j$. 

Using Lemma~\ref{lem:kktnnlad}, Lemma~\ref{lem:extcone} and the fact that $\vp$ is strictly positive, we have 
$\max_{1\leq j\leq n} \vD_i^{\star^T} \vY_j = \max_{j\notin A^t} \vD_i^{\star^T} \vY_j$. Indeed, for all $j\in A^t$ we have
$\vD_i^{\star^T} \vY_j \leq 0$ using Lemma~\ref{lem:kktnnlad} and there is at least one $j=i\notin A^t$ for which
$\vD_i^{\star^T} \vY_j > 0$ using Lemma~\ref{lem:extcone}. Hence the maximum lies in the set $\{j: j \notin A^t\}$.

Further, we have $\max_{j\notin A^t} \vD_i^{\star^T} \vY_j = \max_{j\notin A^t} \vD_i^{\star^T} \vY_A \vC_j 
\leq \max_{j\in (A \setminus A^t)} \vD_i^{\star^T} \vY_j$. The second inequality is the result of the fact that
$\lVert \vC_j\rVert_1 =1$ and $\vC_j\geq 0$. This implies that if there is a unique maximum at a
$j^* = \argmax_{j\notin A^t}\vD_i^{\star^T} \vY_j$, then $\vX_{j^*}$ is an anchor that has not been selected so far.
%\vspace{-2mm}
\end{proof}

\vspace{-2mm}
\noindent {\bf Remarks}: 

(1) For the correctness of Algorithm~\ref{alg:conicl1}, the anchor selection step requires choosing a $\vD_i^\star\in\vD_i$ for which
Lemma~\ref{lem:kktnnlad} and Lemma~\ref{lem:extcone} hold true. 
Here we give a method to find one such $\vD_i^\star$ using linear programming. Using KKT conditions, the $\vD_i^\star$ 
satisfying $-\vX_A^T\vD_i^\star=\vLambda_i\in\reals_+^{|A|}$ is a candidate. We know $\vLambda_{ji}=0$ if $\vH_{ji}>0$ 
and $\vLambda_{ji}>0$ if $\vH_{ji}=0$ (complementary slackness). Let $Z=\{j:\vH_{ji}>0\}$ and $\widetilde{Z}=\{j:\vH_{ji}=0\}$. 
%Let us denote by $\vQ\in\{0,1\}^{|Z|\times |A|}$ a matrix that selects the zero elements of $\vLambda_i$ (i.e., $\vQ\vLambda_i =0$)
%and by $\widetilde{\vQ}\in\{0,1\}^{|\widetilde{Z}|\times |A|}$ a matrix that selects the non-zero elements of $\vLambda_i$ (i.e., $\widetilde{\vQ}\vLambda_i>0$).
%Let us also permute (by a permutation matrix $\vP$) the elements of $\vD_i^\star$
%to get a vector $\vd = \vP\vD_i^\star=[\vu^T \vv^T]^T$ such that its top elements $\vu$ are unknowns (correspond to $\vR_{ji}=0$) and bottom elements
%are known (correspond to $\vR_{ji}\neq 0$). 
Let $I=\{j:\vR_{ji}=0\}$.  % and $\widetilde{I}=\{j:\vR_{ji}\neq 0\}$. 
 Let $\vu$ represent the elements of $\vD_i^\star$ that we need to find, i.e., $\vu=\{\vD_{ji}^\star : j\in I\}$. 
Finding $\vu$ is a feasibility problem that can be solved using an LP. Since there can be
multiple feasible points, we can choose a dummy cost function $\sum_k \vu_k$ (or any other random linear function of $\vu$) 
for the LP. More formally, the LP takes the form: 
%\begin{align
%& ~~~~~~~~~~ \min_{-1\leq\vu\leq 1} \mbf{1}^T\vu, ~~~\text{such that} \\
%& -\vQ\vX_A^T\vP^T[\vu^T \vv^T]^T=0, ~~~ \bar{\vQ}\vX_A^T\vP^T[\vu^T \vv^T]^T<0
%\end{align*}
\begin{align*}
\vspace{-3mm}
& ~~~~~~~~~~ \min_{-1\leq\vu\leq 1} \mbf{1}^T\vu, ~~~\text{such that} \\
& -\vX_A^T\vD_i^\star=\vLambda_i, \vLambda_{ji}=0~\forall~j\in Z, \vLambda_{ji}>0~\forall~j\in \widetilde{Z}
\vspace{-3mm}
\end{align*}
In principle, the number of variables in this LP is the number of zero entries in residual vector $\vR_i$ which can be as large as $m-1$.
In practice, we always have the number of zeros in $\vR_i$ much less than $m$ since we always pick the exterior point 
with maximum $\ell_1$ norm in the Anchor Selection step of Algorithm~\ref{alg:conicl1}. The number of constraints in the LP 
is also very small ($= |A| <r)$. In our implementation, we simply set $\vu=-1$
which, in practice,  almost always satisfies Lemma~\ref{lem:kktnnlad} and
Lemma~\ref{lem:extcone}. The LP is called whenever
Lemma~\ref{lem:extcone} is violated which happens rarely (note that Lemma~\ref{lem:kktnnlad} will never violate with this setting of $\vu$).

(2) If the maximum of Eq.~\ref{eq:selection} occurs at two points $j_1^*$ and $j_2^*$, both these points $\vX_{j_1^*}$ and $\vX_{j_2^*}$
generate the extreme rays of the data cone. Hence both are added to anchor set $A$. 
If the maximum occurs at more than two points, 
some of these are the anchors and others are conic combinations of these anchors.
We can identify the anchors of this subset of points by calling Algorithm~\ref{alg:conicl1} recursively.% and add them to anchor set $A$. 

(3) In Algorithm~\ref{alg:conicl1}, the vector $\vp$ needs to satisfy $\vp^T \vx_i >0, i=1\ldots n$. In our implementation, we simply used $\vp = \mbf{1}+\vdelta \in \reals^m$ where $\vdelta$ is small perturbation vector with entries i.i.d. according to a uniform distribution $\mcal{U}(0,10^{-5})$. This is done to avoid 
the possibility of $\vp$ being collinear with $\vD_i^{\star}$.

\section{Near-separable NMF with Bregman divergence}
\label{sec:bregman}
Let $\phi:\cS\mapsto\mbb{R}$ be a strictly convex function on domain $\cS\subseteq\mbb{R}$
which is differentiable on its non-empty relative interior $ri(\cS)$.  
Bregman divergence is then defined as $D_\phi(x,y) = \phi(x)-\phi(y)-\phi'(y)(x-y)$
where $\phi'(y)$ is the continuous first derivative of $\phi(\cdot)$ at $y$. 
Here we will also assume $\phi'(\cdot)$ to be smooth which is true for most Bregman divergences 
of interest. 
A Bregman divergence is always convex in the first argument. Some instances of Bregman divergence are also
convex in the second argument (e.g., KL divergence). 
For two matrices $\vX$ and $\vY$, we work with divergence of the form $D_\phi(\vX,\vY):=\sum_{ij}D_\phi(\vX_{ij},\vY_{ij})$.  

Here we consider the case when the entries of data matrix $\vX$ are generated from an exponential family distribution
with parameters satisfying the separability assumption, i.e., 
$\vX_{ij}\sim \cp_\phi(\vW^i\vH_j),~\vW\in\reals_+^{m\times r},\vH=[\vI\,\,\vH']\in\reals_+^{r\times n}$
($\vW^i$ and $\vH_j$ denote the $i$th row of $\vW$ and the $j$th column of $\vH$, respectively). 
Every member distribution $\cp_\phi$ of the exponential family has a unique Bregman divergence $D_\phi(\cdot,\cdot)$ associated 
with it~\cite{clusterbregman05},
and solving $\min_\vY D_\phi(\vX,\vY)$ is equivalent to maximum likelihood estimation for parameters $\vY_{ij}$ of the
distribution $\cp_\phi(\vY_{ij})$. %, assuming that $\vX_{ij}$ are generated from it.
Hence, the projection step in Algorithm~\ref{alg:conicl1} is changed to $\vH = \argmin_{\vB\geq 0} D_\phi(\vX,\vX_A\vB)$.
We use the coordinate descent based method of \citet{bregmannmf.12} to solve the projection step. 
To select the anchor columns with Bregman projections $\vX_A\vH$, we modify the selection criteria as 
\begin{align}
j^\star = \argmax_j \frac{(\phi''(\vX_A\vH_i)\odot\vR_i)^T\vX_j}{\vp^T \vX_j}%~\textrm{for any}~i:\|\vR_i\|>0,
\label{eq:selbreg}
\end{align} 
for any $i:\|\vR_i\|>0$, where $\vR=\vX-\vX_A\vH$ and $\phi''(\vx)$ is the vector of second derivatives of $\phi(\cdot)$ evaluated
at individual elements of the vector $\vx$ (i.e., $[\phi''(\vx)]_j=\phi''(\vx_j)$),
and $\odot$ denotes element-wise product of vectors. We can show the following result
regarding the anchor selection property of this criteria. Recall that an anchor is a column that can not be expressed 
as conic combination of other columns in $\vX$.
\begin{thm}
If the maximizer of Eq.~\ref{eq:selbreg} is unique, the data point $\vX_{j^\star}$ added at each iteration in the Selection step, is 
an anchor that has not been selected in one of the previous iterations.
\vspace{-3mm}
\end{thm}
The proof is provided in the Appendix. 
Again, any exterior point $i$ can be chosen to select the next anchor %in pure separable case 
but our simulations show that taking exterior point to be $i = \argmax_k D_\phi(\vX_k,\vX_A\vH_k)$ gives much better performance
under noise than randomly choosing the exterior point.
Note that for the Bregman divergence induced by function $\phi(x)=x^2$, the selection criteria 
of Eq.~\ref{eq:selbreg} reduces to the selection criteria of \xray~ proposed in \cite{xray.icml13}.

Since Bregman divergence is not generally symmetric, it is also possible to have the projection step as
$\vH = \argmin_{\vB\geq 0} D_\phi(\vX_A\vB,\vX)$. In this case, the selection criteria will change to
$j^\star = \argmax_j \frac{(\phi'(\vX_i)-\phi'(\vX_A\vH_i))^T \vX_j}{\vp^T \vX_j}$ for any point $i$ exterior to the current cone,
where $\phi'(\vx)$ operates element-wise on vector $\vx$.
However this variant does not have as meaningful a probabilistic interpretation
as the one discussed earlier. % and hence is not a popular in the NMF
% literature.

\section{Empirical Observations}
\label{sec:synth}
In this section, we present experiments on synthetic and real datasets to demonstrate the
effectiveness of the proposed algorithms under noisy conditions. 
In addition to comparing our algorithms with existing separable NMF methods
\cite{bittorf.12,gillis.12,xray.icml13}, we also benchmark them against
Robust PCA and local-search based low-rank factorization methods, wherever applicable, for 
the sake of providing a more complete picture. 

\begin{figure}[t]
\centering
%\vspace{-2mm}
\includegraphics[clip=true,trim=0 0 0 0,width=0.4\linewidth,height=5cm]{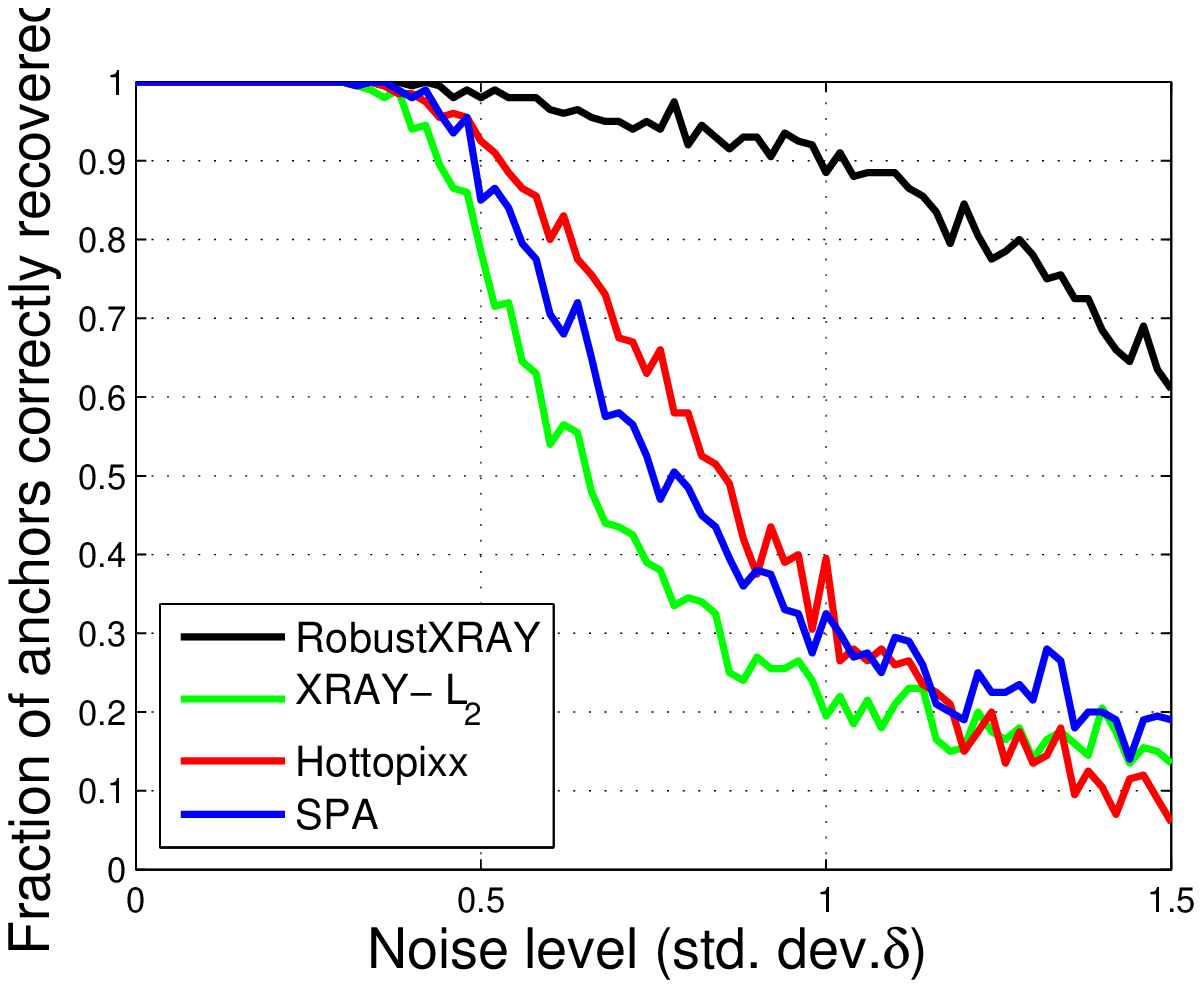}  
\includegraphics[clip=true,trim=0 0 0 0,width=0.4\linewidth,height=5cm]{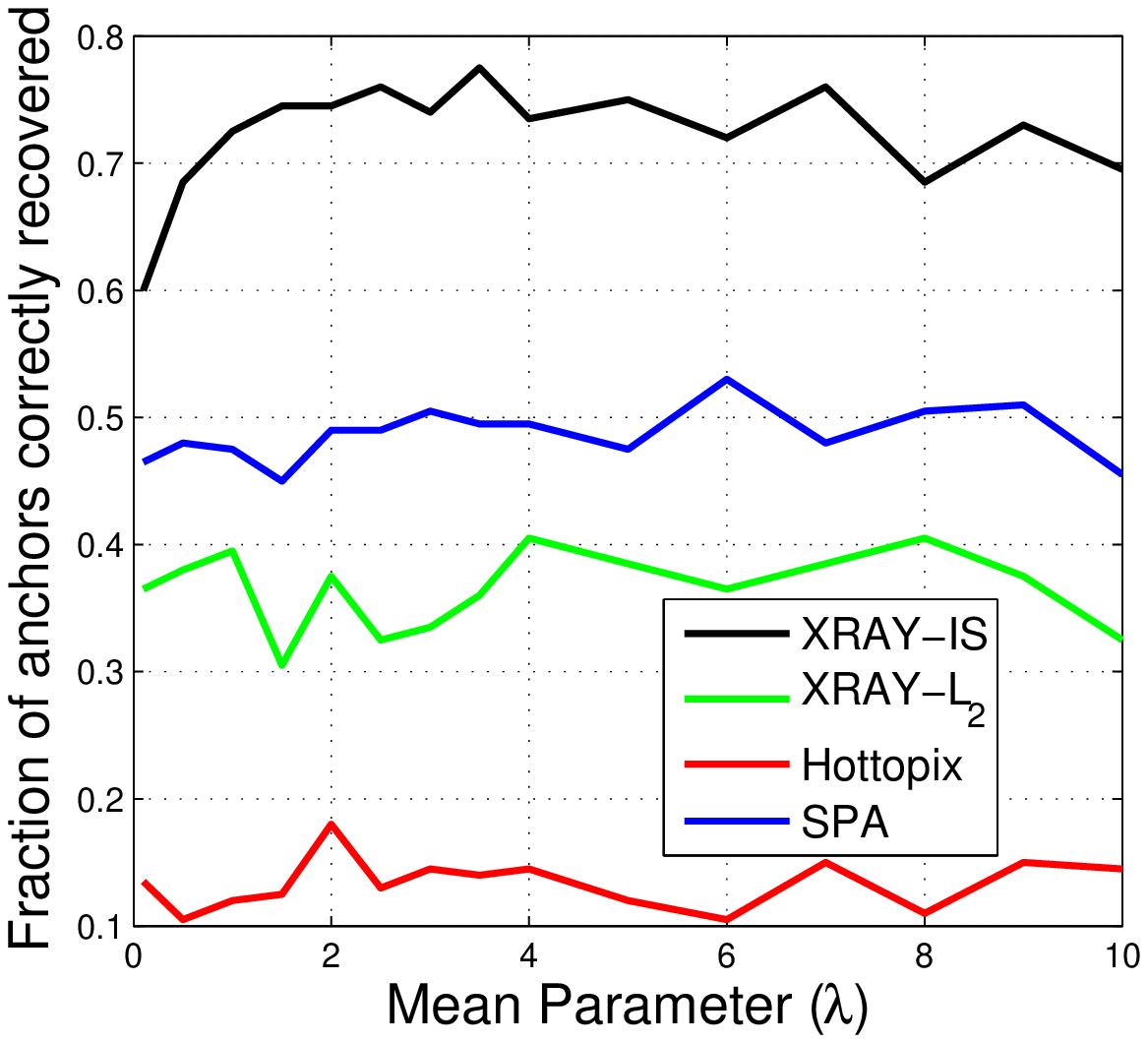}  
\vspace{-3mm}
\caption{\emph{\bf Left:} Anchor recovery rate versus noise level for sparse noise case, \emph{\bf Right:} 
Anchor recovery rate versus mean parameter for data matrix generated from exponential distribution (best viewed in color)}
\label{fig:synthl1is}
%\vspace{-5mm}
\end{figure}

%\begin{figure}[t]
%\centering
%\vspace{-2mm}
%\includegraphics[clip=true,trim=0 0 0 0,width=0.8\linewidth,height=5cm]{figures/synthIS.eps}  
%\vspace{-3mm}
%\caption{Data matrix generated from exponential distribution: anchor recovery rate versus mean parameter (best viewed in color)}
%\label{fig:synthis}
%\vspace{-5mm}
%\end{figure}

\subsection{Anchor recovery under noise}
Here we test the proposed algorithms for recovery of anchors when the separable
structure is perturbed by noise. We compare with methods proposed in \citet{gillis.12}
(abbrv. as SPA for Successive Projection Approximation), \citet{bittorf.12} (abbrv. as \hott)
and \citet{xray.icml13} (abbrv. as \xrayl2). 

First, we consider the case when the separable structure is perturbed by addition of a sparse
noise matrix, i.e., $\vX=\vW\vH+\vN,~\vH=[\vI\,\,\vH']$. Each entry of matrix $\vW\in\reals_+^{200\times 20}$ is
generated i.i.d. from a uniform distribution between $0$ and $1$. The matrix $\vH\in\reals^{20\times 210}$ 
is taken to be $[\vI_{20\times 20}\,\,\vH'_{20\times 190}]$ where each column of $\vH'$ is sampled i.i.d. from a
Dirichlet distribution whose parameters are generated i.i.d. from a uniform distribution between $0$ and $1$.
It is clear from the structure of matrix $\vH$ that first twenty columns are the anchors.
The data matrix $\vX$ is generated as $\vW\vH+\vN$ with $\vN=\max(\vN_1,0)\in\reals_+^{200\times 210}$,
where each entry of $\vN_1$ is generated i.i.d. from a Laplace distribution having zero mean and $\delta$
standard deviation. Since Laplace distribution is symmetric around mean, almost half of the entries
in matrix $\vN$ are $0$ due to the $\max$ operation. The std. dev. $\delta$ is varied from $0$ to $1.5$ with a step size 
of $0.02$. Fig.~\ref{fig:synthl1is} plots the fraction of correctly recovered anchors averaged over $10$ runs for each 
value of $\delta$. The proposed Robust\xray~ (Algorithm~\ref{alg:conicl1}) outperforms all other methods
including \xrayl2 by a huge margin as the noise level increases. This highlights the importance of 
using the right loss function in the projection step that is suitable for the noise model (in this case $\ell_1$ loss of Eq.~\ref{eq:nnlad}).

Next, we consider the case where the non-negative data matrix is generated from an exponential family
distribution other than the Gaussian, i.e., 
$\vX_{ij}\sim \cp_\phi(\vW^i\vH_j),~\vW\in\reals_+^{m\times r},\vH=[\vI\,\,\vH']\in\reals_+^{r\times n}$
($\vW^i$ and $\vH_j$ denote the $i$th row of $\vW$ and the $j$th column of $\vH$, respectively). 
As mentioned earlier,
every member distribution $\cp_\phi$ of the exponential family has a unique Bregman divergence $D_\phi$ associated with it. 
Hence we minimize the corresponding Bregman divergence in the projection step of the algorithm
as discussed in Section~\ref{sec:bregman}, to recover the anchor columns. 
Two most commonly used Bregman divergences are generalized KL-divergence and Itakura-Saito (IS) divergence~\cite{nmfbregman05,clusterbregman05,nmfIS09}
that correspond to Poisson and Exponential distributions, respectively. We do not report results with
generalized KL-divergence here since they were not very informative in highlighting the differences among various algorithms
that are considered. The reason is that Poisson distribution with parameter $\lambda$ has a mean of $\lambda$ and std. dev. of $\sqrt{\lambda}$,
and increasing the noise (std. dev.) actually increases the signal to noise ratio\footnote{Poisson distribution with parameter $\lambda$ 
closely resembles a Gaussian distribution with mean $\lambda$ and std. dev. $\sqrt{\lambda}$, for large values of $\lambda$.}. Hence anchor recovery gets
better with increasing $\lambda$ (perfect recovery after certain value) 
and almost all algorithms perform as well as \xraykl for the full $\lambda$ range.
The anchor recovery results with IS-divergence are shown in Fig.~\ref{fig:synthl1is}. 
The entries of data matrix $\vX$ are generated as 
$\vX_{ij}\sim\exp(\lambda\vW^i\vH_j)~\vW\in\reals_+^{200\times 20},\vH=[\vI_{20\times 20}\,\,\vH']\in\reals_+^{20\times 210}$.
The matrices $\vW$ and $\vH$ are generated as described in the previous paragraph. The parameter $\lambda$ is varied from 
$0$ to $10$ in the steps of $0.5$ and we report the average over $10$ runs for each value of $\lambda$. 
The \xrayis algorithm significantly outperforms other methods including \xrayl2~\cite{xray.icml13} in correctly recovering the 
anchor column indices. The recovery rate does not change much with increasing $\lambda$ since exponential distribution with
mean parameter $\lambda$ has a std. dev. of $\lambda$ and the signal-to-noise ratio practically stays almost same with varying $\lambda$.

\begin{figure}[t]
\centering
%\vspace{-2mm}
\includegraphics[clip=true,trim=0 0 0 0,width=0.4\linewidth,height=5cm]{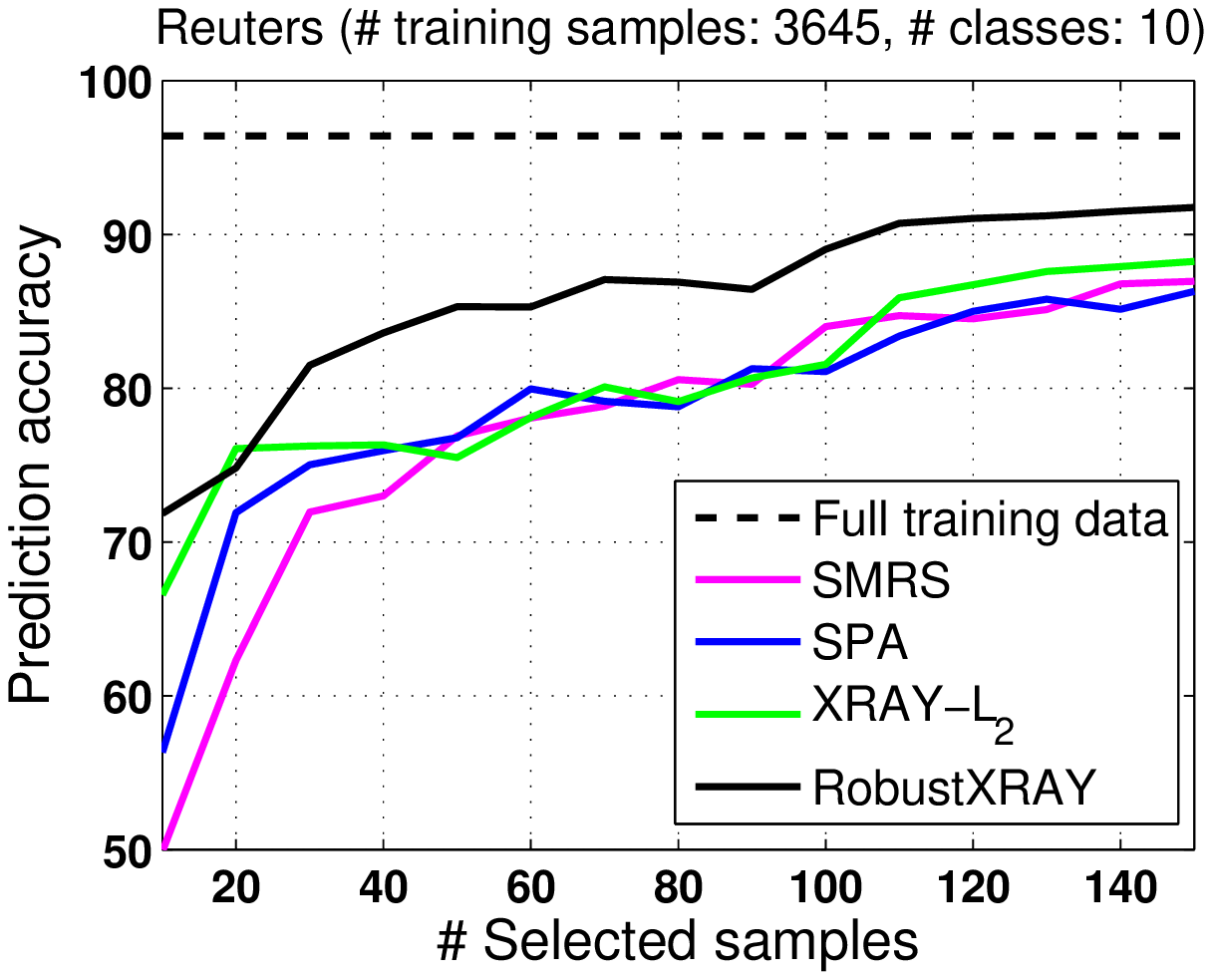}  
\includegraphics[clip=true,trim=0 0 0 0,width=0.4\linewidth,height=5cm]{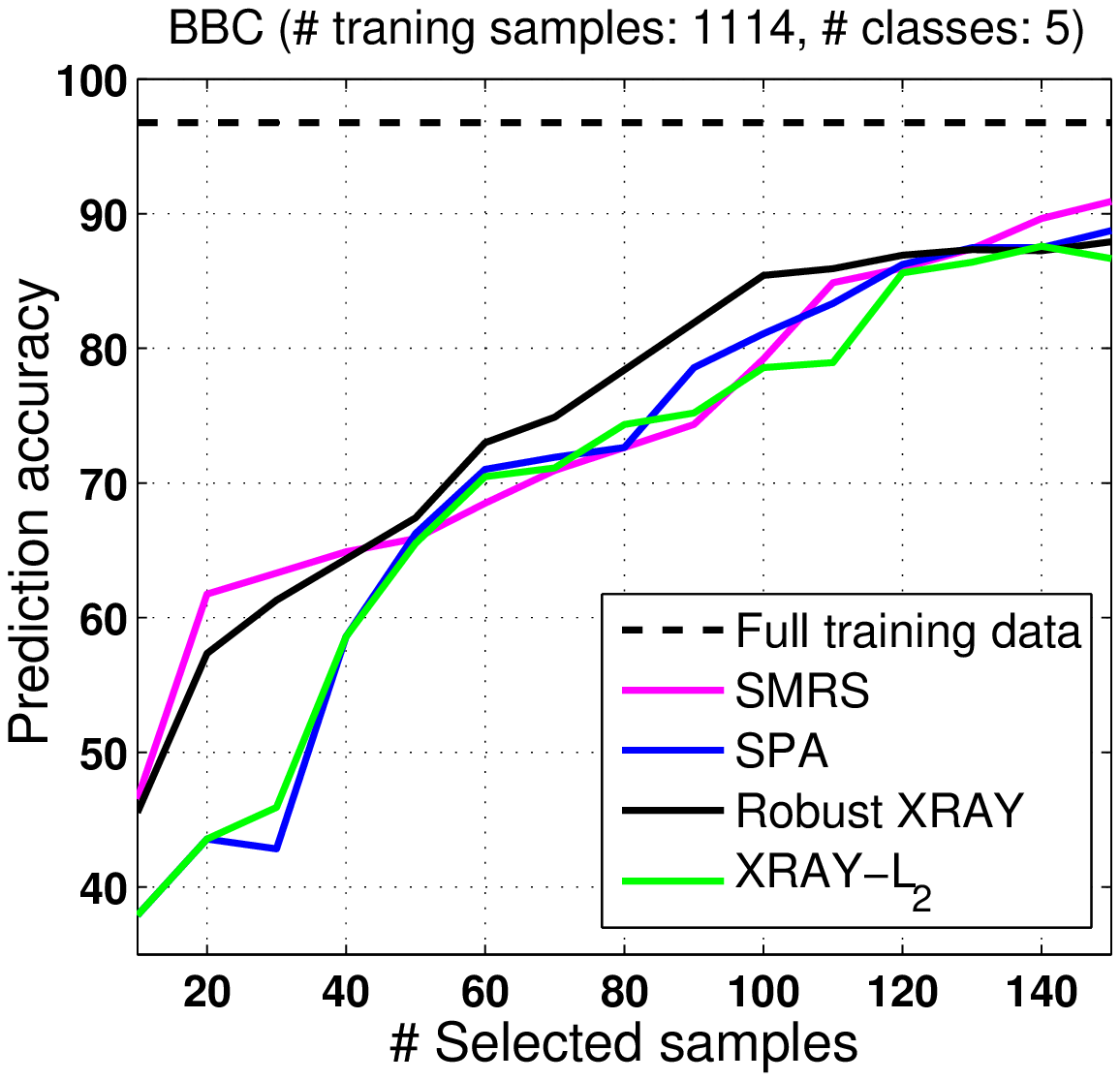}  
\vspace{-3mm}
\caption{Accuracy of SVM trained with selected exemplars on \emph{Reuters data (left)} and \emph{BBC data (right)} (best viewed in color)}
\label{fig:exemp.reutbbc}
%\vspace{-3mm}
\end{figure}

%\begin{figure}[ht]
%\centering
%\vspace{-2mm}
%\includegraphics[clip=true,trim=0 0 0 0,width=0.8\linewidth,height=5cm]{figures/exemplar_bbc_xrayl2max.eps}  
%\vspace{-3mm}
%\caption{Accuracy of SVM trained with selected exemplars on \emph{BBC data} (best viewed in color)}
%\label{fig:exemp.bbc}
%\vspace{-5mm}
%\end{figure}

\subsection{Exemplar Selection}
\label{sec:exemplar}
The problem of exemplar selection is concerned with finding a few representatives from a dataset that can summarize the 
dataset well. Exemplar selection can be used in many applications including summarizing a video sequence, selecting representative images or
text documents (e.g., tweets) from a collection, etc. 
If $\vX$ denotes the data matrix where each \emph{column} is a data point, the exemplar selection problem translates to
selecting a few columns from $\vX$ that can act as representatives for all the columns. The separable NMF algorithms can be used for 
this task, working under the assumption that all data points (columns of $\vX$) can be expressed as non-negative linear combinations 
of the exemplars (the anchor columns). To be able to compare the quality of the selected exemplars by different algorithms in an objective manner,
we test them on a classification task (assuming that every data point has an associated label). 
We randomly partition the data in training and test sets, and use only training set in selecting the exemplars. 
We train a multiclass SVM classifier~\cite{liblinear} with the selected exemplars and look at its accuracy on the held-out test set. 
The accuracy of the classifier trained with the full training set is taken as a benchmark and is also reported. 
We also compare with \citet{elhamifar.cvpr12} who recently proposed a method for exemplar selection, named as
Sparse Modeling Representative Selection (SMRS). 
They assume that the data points can be expressed as a convex linear combination
of the exemplars and minimize $\lVert\vX-\vX\vC\rVert_\fro^2 + \lambda\lVert\vC\rVert_{1,2}$ s.t. $\mbf{1}^T\vC=\mbf{1}^T$. 
The columns of $\vX$ corresponding to the non-zero rows of $\vC$ are selected as exemplars. 
We use the code provided by the authors for SMRS. 
There are multiple possibilities for anchor selection criteria in 
the proposed Robust\xray~ and \xrayl2~\cite{xray.icml13}
and we use \emph{max} criterion for both the algorithms. 

We report results with two text datasets: Reuters~\cite{reuters} and BBC~\cite{bbc.greene06icml}. 
We use a subset of Reuters data corresponding to the most frequent $10$ classes which amounts to 7285 documents
and 18221 words ($\vX\in\reals_+^{18221\times 7285}$). The BBC data consists of 2225 documents and 9635 words
with 5 classes ($\vX\in\reals_+^{9635\times 2225}$). For both datasets, we evenly split the documents into training and test set, and select
the exemplars from the training set using various algorithms. 
Fig.~\ref{fig:exemp.reutbbc} and Fig.~\ref{fig:exemp.reutbbc} show the plot of SVM accuracy on the test set against the number of selected exemplars
that are used for training the classifier. The number of selected anchors is varied from 10 to 150 in the steps of 10.
The accuracy using the full training set is also shown (dotted black line). 
For Reuters data, the proposed Robust\xray~ algorithm outperforms other methods by a significant margin
for the whole range of selected anchors. All methods seem to perform comparably on BBC data. 
An advantage of SPA and \xray~ family of methods is that there is no need for a cleaning step to remove 
near-duplicate exemplars as needed in SMRS~\cite{elhamifar.cvpr12}. Another advantage is of computational speed --
in all our experiments, SPA and \xray~ methods are about 3--10 times faster than SMRS. 
It is remarkable that even 
a low number of selected exemplars give reasonable classification accuracy for all methods -- SMRS gives $50\%$ accuracy for Reuters data using
$10$ exemplars (on average 1 training sample per class) while Robust\xray~ gives more than 70\%. 

\begin{figure*}[ht]
\newlength{\figwidthsvm}
\setlength{\figwidthsvm}{0.33\textwidth}
\begin{center}
\hspace{-5mm}\begin{minipage}{0.33\textwidth}
\centering
\includegraphics[width=1.1\columnwidth, height=5cm]{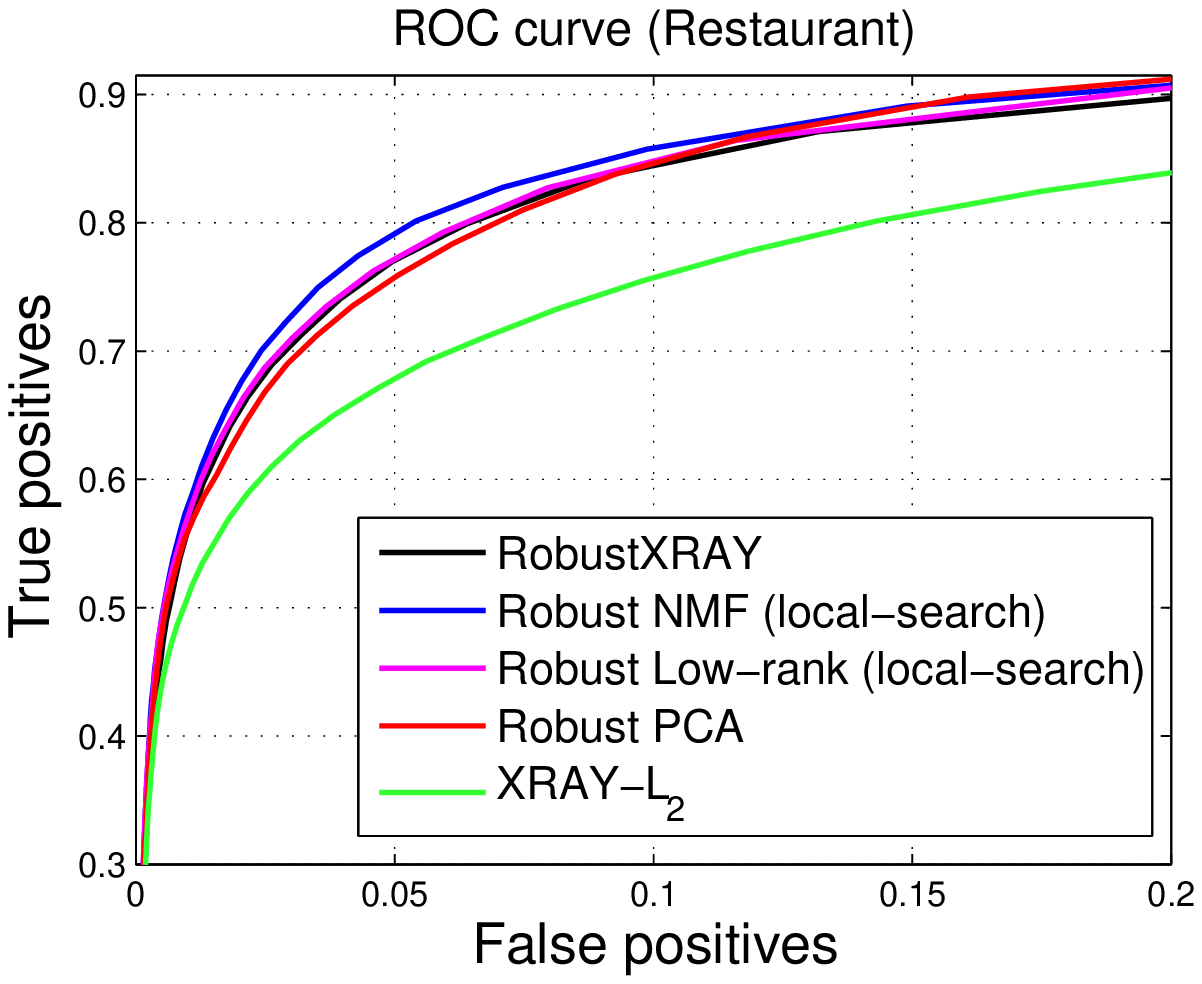}
\end{minipage}
\begin{minipage}{0.33\textwidth}
\centering
\includegraphics[width=1.1\columnwidth, height=5cm]{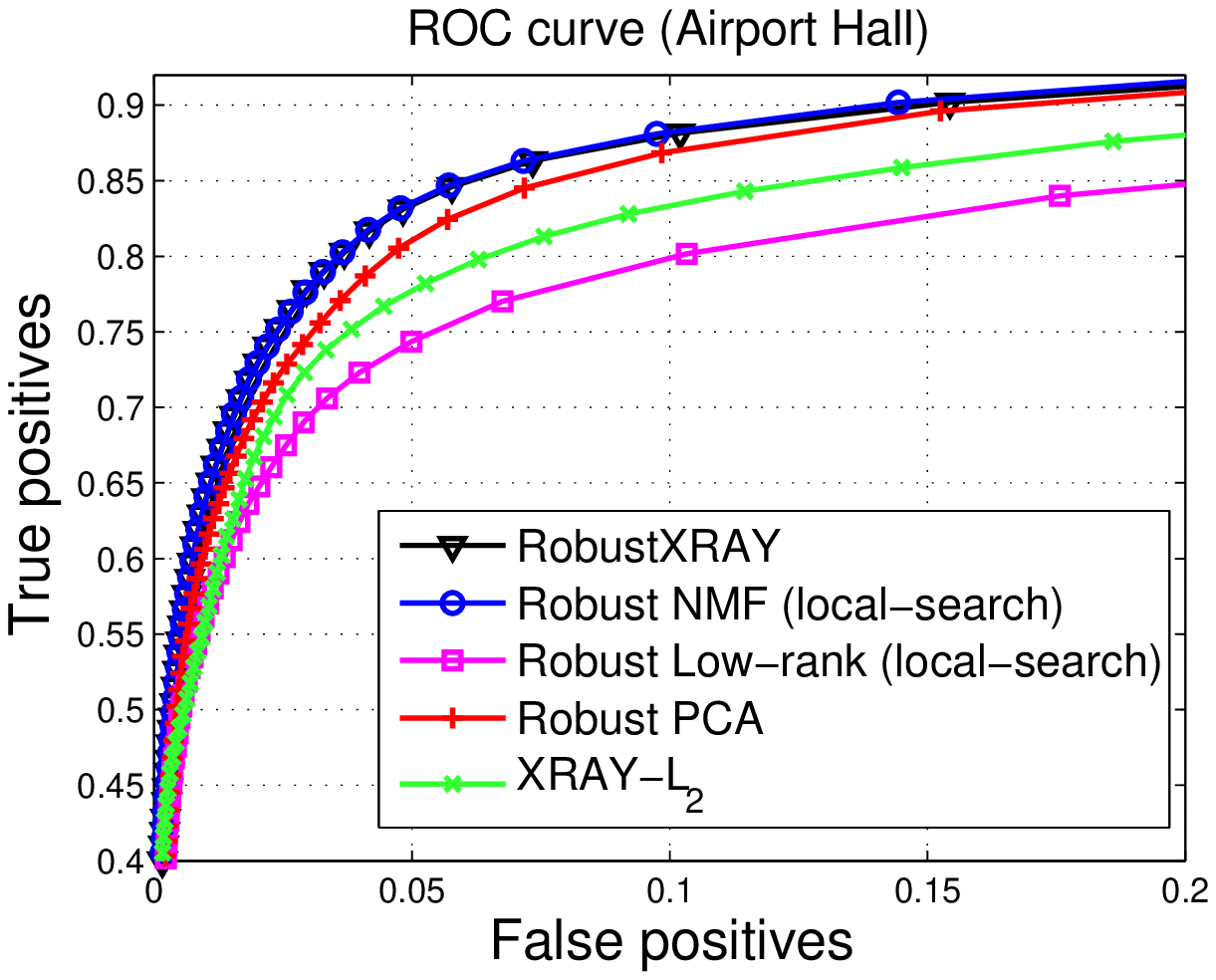}
\end{minipage}
\begin{minipage}{0.33\textwidth}
\centering
\includegraphics[width=1.1\columnwidth, height=5cm]{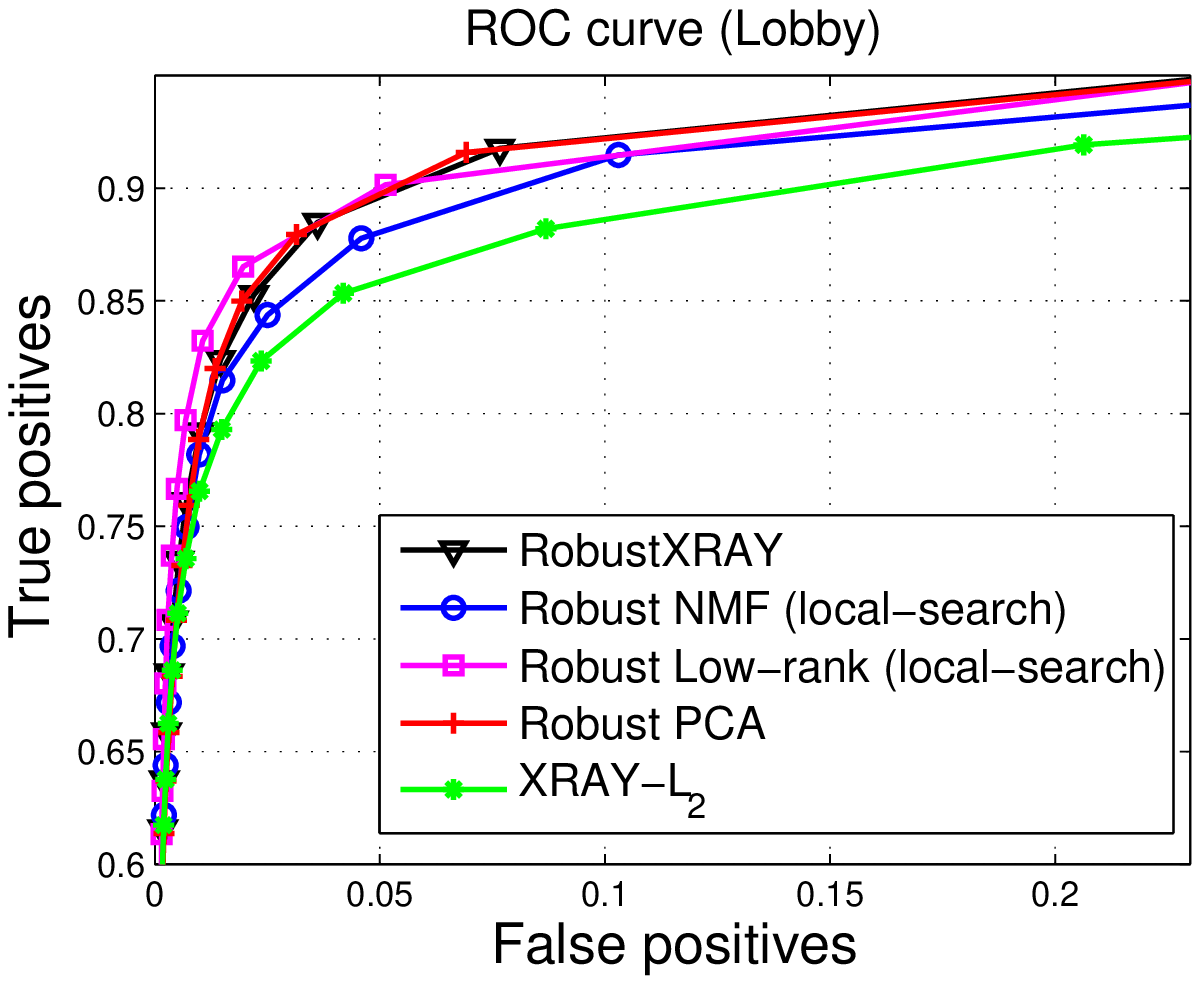}
\end{minipage}
\vspace{-1mm}
\caption{Foreground-background separation: ROC curves with various methods for \emph{Restaurant}, \emph{Airport Hall} 
and \emph{Lobby} video sequences. The ranges for X and Y axes are chosen to better highlight the differences 
among the ROC curves. (best viewed in color)}
\label{fig:fgbg}
\end{center}
\vspace{-4mm}
\end{figure*}

\subsection{Foreground-background Separation}
\label{sec:fgbg}
%plot roc_bootstrap: robustXray(rank 19), local-search robust NMF(rank 11), local-search robust low rank (rank 14), RPCA (0.001), XRAY-L2 (rank 3)
%plot roc_hall: robustXray(rank 1), local-search NMF(rank 2), local-search robust low-rank (rank5), XRAY-L2(rank1), RPCA (0.001)
%plot roc_lobby: robustXray(rank 3), local-search NMF(rank 2), local-search robust low-rank(rank5), Xray-l2(rank 3), RPCA (0.001)
In this section, we consider the problem of foreground-background separation in video. The camera position is assumed to be \emph{almost} fixed
throughout the video. In all video frames, the camera captures the background scene superimposed with a \emph{limited} 
foreground activity (e.g., movement of people or objects). Background is assumed to be stationary or slowly varying
across frames (variations in illumination and shadows due to lighting or time of day) while foreground is assumed to be
composed of objects that move across frames but span only a few pixels. If we vectorize all video frames and stack
them as rows to form the matrix $\vX$, the foreground-background separation problem can be modeled as decomposing $\vX$
into a low-rank matrix $\vL$ (modeling the background) and a sparse matrix $\vS$ (modeling the foreground). 
%This is equivalent to solving the problem $\min_{\vL,\vS} \text{rank}(L) +
% \lambda \|S\|_1,$ s.t. $\vX=\vL+\vS$, which is NP-hard. 
%Robust PCA~\cite{candes.rpca} makes this problem tractable by relaxing the rank
% penalty in the objective to nuclear-norm %and $\ell_1$ norm respectively
%and is a popular method for this problem. We can also use NMF to tackle this
%% problem by modeling background as a low non-negative rank matrix and
%% foreground as a sparse matrix, i.e. $\min_{\vW\geq 0,\vH\geq 0} \|\vS\|_1$
%s.t. $\vX=\vW\vH+\vS$ or equivalently $\min_{\vW\geq 0,\vH\geq 0}
% \|\vX-\vW\vH\|_1$. We term this problem as \emph{Robust NMF}. Since NMF is NP-hard~\cite{Vavasis.09},
%one possibility is use local-search based optimization for this
% problem~\cite{manhattannmf.12}.
%However, the solution of local-search can be sensitive to initialization. 
%Here we consider the other possibility of using separability assumption to make
% the NMF tractable and evaluate near-separable NMF with $\ell_1$ loss (Robust\xray) on the task of foreground-background separation. 

{\bf Connection to Median Filtering:} Median filtering is one of the most commonly used background modeling techniques~\cite{sen2004robust},
which simply models the background as the pixel-wise median of the video frames. The assumption is that each pixel location
stays in the background for more than half of the video frames. Consider the NMF of inner-dimension 1: $\min_{\vw\geq 0,\vh\geq 0} \|\vX-\vw\vh\|_1$.
If we constrain the vector $\vw$ to be all ones, the solution $\vh^* = \argmin_{\vw=\mbf{1},\vh\geq 0} \|\vX-\vw\vh\|_1$ is nothing
but the element-wise median of all rows of $\vX$. More generally, if $\vw$ is constrained to be such that $\vw_i=c>0$ $\forall$ $i$,
the solution $\vh^*$ is a scaled version of the element-wise median vector. 
Hence Robust NMF under this very restrictive setting is equivalent to median filtering on the video frames, and we can hope
that loosening this assumption and allowing for higher inner-dimension in the factorization can help in modeling more variations
in the background. 

We use three video sequences for evaluation: \emph{Restaurant}, \emph{Airport Hall} and \emph{Lobby}
 \cite{fgbg.9datasets}. 
{Restaurant} and {Airport Hall} are videos taken at a buffet restaurant and at a hall of an airport, respectively.
The lighting are distributed from the ceilings and signiﬁcant shadows of moving persons cast on the ground surfaces from different 
directions can be observed in the videos. {Lobby} video sequence was captured from a lobby in an
office building and has background changes due to lights being switched on/off. 
The ground truth (whether a pixel belongs to foreground or background) is also available for these video sequences
and we use it to generate the ROC curves. 
We mainly compare Robust\xray~ with Robust PCA which is widely considered
state of the art methodology for this task in the Computer Vision community.
In addition, we also compare with two local-search based approaches:
(i) \emph{Robust NMF (local-search)} which solves $\min_{\vW\geq 0,\vH\geq 0} \|\vX-\vW\vH\|_1$ using local search,
and (ii) \emph{Robust Low-rank (local-search)} which solves $\min_{\vW,\vH} \|\vX-\vW\vH\|_1$ using local search.
We use an ADMM based optimization procedure for both these local-search methods. 
We also show results with \xrayl2 of \citet{xray.icml13} to highlight the importance of having near-separable NMFs with
$\ell_1$ loss for this problem. For both \xrayl2 and Robust\xray~, we do 1 to 2 refitting steps to refine the solution
(i.e., solve $\vH=\min_{\vB\geq 0}\|\vX-\vX_A\vB\|_1$ then solve $\vW=\min_{\vC\geq 0}\|\vX-\vC\vH\|_1$).
For all the methods, we do a grid search on the parameters (inner-dimension or rank parameter for the factorization methods
and $\lambda$ parameter for Robust PCA) and report the best results for each method. 

Fig.~\ref{fig:fgbg} shows the ROC plots for the three video datasets. For Restaurant data, all robust methods
(those with $\ell_1$ penalty on the foreground) perform almost similarly. 
For Airport Hall data, Robust\xray~ is tied with local-search based Robust NMF and these two are better than other methods. 
Surprisingly, \xrayl2 performs better than local-search based Robust Low-rank which might be due to bad initialization. 
For Lobby data, local-search based Robust low-rank, Robust PCA and Robust\xray~ perform almost similarly, and are better than
local-search based Robust NMF. The results on these three datasets show that Robust\xray~ is a promising method for
the problem of foreground-background separation which has a huge advantage over Robust PCA in terms of speed. Our MATLAB
implementation was at least 10 times faster than the inexact Augmented Lagrange Multiplier (i-ALM) implementation of 
\citet{lin2010alm}.

\section{Conclusion and Future Work} We have proposed generalized conical hull
algorithms to extend near-separable NMFs to robust ($\ell_1$) loss function and
Bregman divergences. Empirical results on exemplar selection and video
background-foreground modeling problems suggest that this is a promising
methodology. Avenues for future work include formal theoretical analysis of
noise robustness and applications to online settings.

\sloppy
\small
\bibliography{fastconicalhull}
\bibliographystyle{icml2014}

\normalsize
\appendix
\section{Near-separable NMF with Bregman divergence}
Let $A$ be the set of anchors selected so far by the algorithm.
Let $\phi(\cdot)$ be the strictly convex function that induces the Bregman divergence 
$D_\phi(x,y)=\phi(x)-\phi(y)-\phi'(y)(x-y)$.
For two matrices $\vX$ and $\vY$, we consider the Bregman divergence of the form
$D_\phi(\vX,\vY):=\sum_{ij}D_\phi(\vX_{ij},\vY_{ij})$. We make the following assumptions for the proofs
in this section which are satisfied by almost all the Bregman divergences of interest:

\emph{Assumption 1:} The first derivative of $\phi$, $\phi'$, is smooth.  

\emph{Assumption 2:} The second derivative of $\phi$, $\phi''$, is positive at all nonzero points,
i.e., $\phi''(x)>0$ for $x\neq 0$. Note that strict convexity does not necessarily imply $\phi''(x)>0$ for all $x$
while the converse is true. 

Here we consider the projection step 
\begin{align}
\vH = \argmin_{\vB\geq 0} D_\phi(\vX,\vX_A\vB),
\label{eq:projbreg}
\end{align}
and the following selection criteria to identify the next anchor column:
\begin{align}
j^\star = \argmax_j \frac{(\phi''(\vX_A\vH_i)\odot\vR_i)^T\vX_j}{\vp^T \vX_j}%~\textrm{for any}~i:\|\vR_i\|>0,
\label{eq:selbregapp}
\end{align} 
for any $i:\|\vR_i\|>0$, where $\vR=\vX-\vX_A\vH$ and $\phi''(\vx)$ is the vector of second derivatives of $\phi(\cdot)$ evaluated
at individual elements of the vector $\vx$ (i.e., $[\phi''(\vx)]_j=\phi''(\vx_j)$),
$\vp$ is any strictly positive vector not collinear with $(\phi''(\vX_A\vH_i)\odot\vR_i)$, and $\odot$ denotes element-wise product of vectors. 
For a matrix $\vX$, $\vX_{ij}$ denotes $ij$th element, $\vX_i$ denotes $i$th column and $\vX_A$ denotes the columns of $\vX$ indexed by set $A$.
$\vX_{A_{ij}}$ denotes $ij$th column of matrix $\vX_A$. 

Here we show the following result
regarding the anchor selection property of Eq.~\ref{eq:selbregapp}. Recall that an anchor is a column that can not be expressed 
as conic combination of other columns in $\vX$.
\begin{thm}
If the maximizer of Eq.~\ref{eq:selbregapp} is unique, the data point $\vX_{j^\star}$ added at each iteration in the Selection step, is 
an anchor that has not been selected in one of the previous iterations.
\label{thm:breg}
\end{thm}
The proof of this theorem follows the same style as the proof of Theorem $3.1$ in the main paper. 
We need the following lemmas to prove Theorem~\ref{thm:breg}.

\begin{lem}
Let $\vR$ be the residual matrix obtained after Bregman projection of columns of $\vX$ onto the current cone.
Then, $(\phi''(\vX_A\vH)\odot\vR)^T\vX_A \leq 0$, where $\vX_A$ are anchor columns selected so far by the algorithm.
\label{lem:kktbreg}
\end{lem}
\begin{proof}
Residuals are given by $\vR = \vX - \vX_A \vH$, where $\vH = \argmin_{\vB\geq 0} D_\phi(\vX,\vX_A\vB)$. \\
Forming the Lagrangian for Eq.~\ref{eq:projbreg}, we get
$\mcal{L}(\vB,\vLam) = D_\phi(\vX,\vX_A\vB) - tr(\vLam^T\vB)$, where the matrix $\vLam$ contains
the non-negative Lagrange multipliers.

At the optimum $\vB = \vH$, we have $\nabla\mcal{L}=0$ which means
\begin{align*}
&~~~~~~ \frac{\partial}{\partial\vB_{mn}} \biggl. D_\phi(\vX,\vX\vB)\biggr|_{\vB=\vH} = \vLambda_{mn} \\
& \Rightarrow \sum_i \biggl[-\vX_{A_{im}}\phi'((\vX_A\vH)_{in}) - \vX_{A_{im}}\phi''((\vX_A\vH)_{in}) (\vX_{in} \biggr.  \\
&-\biggl. (\vX_A\vH)_{in}) + \phi'((\vX_A\vH)_{in}) \vX_{A_{im}} \biggr]  = \vLambda_{mn}\\
& \Rightarrow \sum_i -\biggl[\vX_{A_{im}}\phi''((\vX_A\vH)_{in}) (\vX_{in}-(\vX_A\vH)_{in})\biggr] = \vLambda_{mn} \\
& \Rightarrow \left[\phi''(\vX_A\vH)\odot (\vX-\vX_A\vH)\right]_n^T \vX_{A_m} = -\vLambda_{mn} \\
& \Rightarrow \left[\phi''(\vX_A\vH)\odot (\vX-\vX_A\vH)\right]^T\vX_A = -\vLambda^T \leq 0,
\end{align*}
\noindent where $\phi''(\cdot)$ is the second derivative of $\phi(\cdot)$ that operates
element-wise on the argument (vector or matrix). 
\end{proof}

\begin{lem}
For any point $\vX_i$ exterior to the current cone, we have $(\phi''(\vX_A\vH)\odot\vR)_i^T \vX_i > 0$.
\label{lem:extconeapp}
\end{lem}
\begin{proof}
For a vector $\vv>0$ and any vector $\vz$, we have $(\vv\odot\vz)^T\vz = \sum_i\vv_i \vz_i^2 > 0$.
Taking $\vv$ to be $\phi''(\vX_A\vH_i)$ and $\vz$ to be $\vR_i$, we have
\begin{align}
\begin{split}
& (\phi''(\vX_A\vH_i)\odot\vR_i)^T \vR_i > 0 \\
& \Rightarrow (\phi''(\vX_A\vH_i)\odot\vR_i)^T (\vX_i -\vX_A\vH_i) > 0
\end{split}
\label{eq:extcone1}
\end{align}
By complementary slackness condition at the optimum, we have $\vLambda_{ji}\vH_{ji}=0$.
From the KKT condition in the proof of previous lemma, we have $\left[\phi''(\vX_A\vH_i)\odot \vR_i\right]^T\vX_A = -\vLambda_i^T$.
Hence 

$\left[\phi''(\vX_A\vH_i)\odot \vR_i\right]^T\vX_A\vH_i = -\vLambda_i^T\vH_i = 0$.

Hence Eq~\ref{eq:extcone1} reduces to 

$(\phi''(\vX_A\vH_i)\odot\vR_i)^T \vX_i > 0$.
\end{proof}

Using the above two lemmas, we can prove Theorem~\ref{thm:breg} as follows.

\begin{proof}
Denoting $\vD_i^\star := (\phi''(\vX_A\vH_i)\odot\vR_i)$
in the proof of Theorem 3.1 of the main paper, all the statements of the proof directly apply in the light of 
Lemma~\ref{lem:kktbreg} and Lemma~\ref{lem:extconeapp}. 
\end{proof}

\end{document}